\documentclass[a4paper]{article}
\usepackage{RR}
\usepackage{hyperref}

\usepackage{amssymb}
\usepackage{amsmath}

\newtheorem{theorem}{Theorem}
\newtheorem{lemma}{Lemma}
\newtheorem{problem}{Problem}
\newtheorem{proposition}{Proposition}
\newtheorem{definition}{Definition}
\newtheorem{remark}{Remark}

\newcommand{\halmos}{\rule[-0.4mm]{2mm}{2mm}}
\newenvironment{proof}{\par\noindent{\bf Proof\ }}{\hfill\halmos\\[2mm]}
\def\@begintheorem#1#2{\tmpitemindent\itemindent\topsep 0pt\rm\trivlist
    \item[\hskip \labelsep{\indent\it #1\ #2: foo}]\itemindent\tmpitemindent}
\def\@opargbegintheorem#1#2#3{\tmpitemindent\itemindent\topsep 0pt\rm \trivlist
    \item[\hskip\labelsep{\indent\it #1\ #2\ foofoo \rm(#3)}]\itemindent\tmpitemindent}
\def\@endtheorem{\endtrivlist\unskip}
\newcommand{\ps}[1]{ \langle{{#1}}\rangle }

\newcommand{\ieg}{\left[\hspace{-0.9ex}\left[\hspace{0.5ex}}
\newcommand{\ied}{\hspace{0.5ex} \right]\hspace{-0.9ex}\right]}

\RRdate{\today} 
\RRauthor{Emmanuel Monfrini\thanks{UMR~7503-UHP}
  \and
Yann Guermeur\thanks{UMR~7503-CNRS}}
\authorhead{Monfrini \& Guermeur} 
\RRtitle{Une SVM multi-classe à coût quadratique\\}
\RRetitle{
A Quadratic Loss Multi-Class SVM} 
\titlehead{A Quadratic Loss Multi-Class SVM} 
\RRresume{La mise en \oe uvre d'une machine à vecteurs support
requiert la détermination des valeurs de deux types
d'hyper-paramètres~: le paramètre de ``marge douce'' $C$ et les
paramètres du noyau. Pour effectuer cette tâche de sélection de
modèle, la méthode de choix est la validation croisée. Sa variante
``leave-one-out'' est connue pour fournir un estimateur de
l'erreur en généralisation presque sans biais. Son défaut premier
réside dans le temps de calcul qu'elle nécessite. Afin de
surmonter cette difficulté, plusieurs majorants de
l'erreur ``leave-one-out'' de la SVM calculant des dichotomies ont
été proposés. La plus populaire de ces bornes supérieures est probablement
la borne ``rayon-marge''. Elle s'applique à la version à marge
dure de la machine, et par extension à la variante dite ``de norne
$2$''. Ce rapport introduit une M-SVM ``à coût quadratique'',
la $\text{M-SVM}^2$, comme une extension directe de la SVM de norne $2$ au
cas multi-classe. Pour cette machine, une borne ``rayon-marge''
généralisée est ensuite établie.}

\RRabstract{Using a support vector machine requires to set two
types of hyperparameters: the soft margin parameter $C$ and the
parameters of the kernel. To perform this model selection task,
the method of choice is cross-validation.
Its leave-one-out variant is known to produce an estimator
of the generalization error which is almost unbiased.
Its major drawback rests in its time requirement.
To overcome this difficulty, several upper
bounds on the leave-one-out error of the pattern recognition SVM
have been derived. Among those bounds, the most popular
one is probably the radius-margin bound. It applies to the hard margin
pattern recognition SVM, and by extension to the $2$-norm SVM.
In this report, we introduce a quadratic loss M-SVM,
the $\text{M-SVM}^2$, as a direct extension of the $2$-norm SVM
to the multi-class case.
For this machine, a generalized radius-margin bound is then established.} 
\RRmotcle{M-SVM, sélection de modèle, erreur ``leave-one-out'',
borne ``rayon-marge''.}
\RRkeyword{M-SVMs, model selection, leave-one-out error,
radius-margin bound.}

\begin{document}
\makeRR

\section{Introduction}

Using a support vector machine (SVM) \cite{BosGuyVap92,CorVap95}
requires to set two types of
hyperparameters: the soft margin parameter $C$ and the parameters
of the kernel. To perform this model selection task, several
approaches are available (see for instance \cite{HasTibFri01,Mas03}).
The solution of choice
consists in applying a cross-validation procedure.
Among those procedures,
the leave-one-out one appears especially attractive, since
it is known to produce an estimator
of the generalization error which is almost unbiased \cite{LunBra69}.
The seamy side of things is that it is highly time consuming.
This is the reason why, in
recent years, a number of upper bounds on the leave-one-out error
of pattern recognition SVMs have been proposed in literature
(see \cite{ChaVapBouMuk02} for a survey).
Among those bounds, the tightest one is the span bound \cite{VapCha00}.
However, the results of Chapelle and co-workers presented in
\cite{ChaVapBouMuk02} show that another bound, the radius-margin one
\cite{Vap98}, achieves equivalent performance for model
selection while being far simpler to compute. This is the reason why
it is currently the most popular bound.
It applies to the hard margin
machine and, by extension, to the $2$-norm SVM
(see for instance Chapter~7 in \cite{ShaCri04}).

In this report, a multi-class extension of the $2$-norm SVM
is introduced. This machine, named $\text{M-SVM}^2$, 
is a quadratic loss multi-class SVM, i.e., a multi-class SVM (M-SVM)
in which the $\ell_1$-norm on the vector of slack variables has been
replaced with a quadratic form. The standard M-SVM on which it is based
is the one of Lee, Lin and Wahba \cite{LeeLinWah04}. As the $2$-norm SVM,
its training algorithm is equivalent to the training algorithm
of a hard margin machine obtained by a simple change of kernel.
We then establish a generalized radius-margin bound on the leave-one-out
error of the hard margin version of the M-SVM of Lee, Lin and Wahba.

The organization of this paper is as follows.
Section~\ref{sec:M-SVM} presents the multi-class SVMs, by
describing their common architecture and the general form
taken by their different training algorithms. It focuses
on the M-SVM of Lee, Lin and Wahba.
In Section~\ref{sec:SUR}, the $\text{M-SVM}^2$ is introduced as a particular case
of quadratic loss M-SVM.
Its connection with the hard margin version of the M-SVM of Lee, Lin and Wahba
is highlighted, as well as the fact that it constitutes a multi-class
generalization of the $2$-norm SVM.
Section~\ref{sec:RM_bound} is devoted to the formulation
and proof of the corresponding multi-class radius-margin bound.
At last, we draw conclusions and outline our ongoing research 
in Section~\ref{sec:conclusion}.

\newpage{}


\section{Multi-Class SVMs}

\label{sec:M-SVM}

\subsection{Formalization of the learning problem}
\label{subsec:formalization}

We are interested here in multi-class pattern recognition
problems. Formally, we consider the case of $Q$-category
classification problems with $3 \leq Q < \infty$, but our results
extend to the case of dichotomies. Each
object is represented by its description $x \in {\cal X}$ and the
set ${\cal Y}$ of the categories $y$ can be identified with the
set of indexes of the categories: $\ieg 1, Q \ied$.
We assume that the link between
objects and categories can be described by an unknown probability
measure $P$ on the product space ${\cal X} \times {\cal Y}$.
The aim of the learning problem consists in selecting in
a set ${\cal G}$ of functions $g = \left ( g_{k} \right )_{1 \leq k \leq Q}$
from ${\cal X}$ into $\mathbb{R}^{Q}$
a function classifying data in an optimal way. The criterion of optimality
must be specified.
The function $g$ assigns $x \in {\cal X}$ to the category $l$ if and only
if $g_l(x) > \max_{k \neq l} g_k(x)$. In case of ex \ae quo, $x$ is assigned
to a dummy category denoted by $*$.
Let $f$ be the decision function (from ${\cal X}$
into ${\cal Y} \bigcup \left \{ * \right \}$)  associated with $g$.
With these definitions at hand,
the objective function to be minimized is the probability of error
$P \left ( f \left ( X \right ) \neq Y \right )$. The optimization process,
called {\em training}, is based on empirical data. More precisely, we assume
that there exists a random pair 
$\left ( X,Y \right ) \in {\cal X} \times {\cal Y}$, distributed according to $P$,
and we are provided with a $m$-sample
$D_m = \left( \left ( X_i, Y_ i \right) \right)_{1\leq i\leq m}$
of independent copies of $\left ( X, Y \right )$.

There are two questions raised by such problems: how to properly
choose the class of functions ${\cal G}$ and how to determine the
best candidate $g^*$ in this class, using only $D_m$. This report addresses
the first question, named {\em model selection}, in the particular case
when the model considered is a M-SVM. The second question,
named {\em function selection}, is addressed for instance in \cite{Gue07b}.

\subsection{Architecture and training algorithms}
\label{sec:architecture}

M-SVMs, like all the SVMs, belong to the family of kernel machines.
As such, they operate on a class of functions induced by a 
positive semidefinite (Mercer) kernel. This calls for the formulation
of some definitions and propositions.

\begin{definition}[Positive semidefinite kernel]
A {\em positive semidefinite kernel} $\kappa$ on the set ${\cal X}$
is a continuous and symmetric function
$\kappa: {\cal X}^2 \to \mathbb{R}$ verifying:
$$
\forall n \in \mathbb{N}^*, \;
\forall \left ( x_i \right )_{1 \leq i \leq n} \in {\cal X}^n, \;
\forall \left ( a_i \right )_{1 \leq i \leq n} \in \mathbb{R}^n, \;
\sum_{i=1}^n \sum_{j=1}^n
a_i a_j \kappa \left ( x_i, x_j \right) \geq 0.
$$
\end{definition}

\begin{definition}[Reproducing kernel Hilbert space \cite{BerTho04}]
Let $\left ( \mathbf{H}, \ps{\cdot,\cdot}_{\mathbf{H}} \right)$
be a Hilbert space of functions
on ${\cal X}$ (${\mathbf{H} \subset \mathbb{R}^{{\cal X}}}$).
A function $\kappa: {\cal X}^2 \to \mathbb{R}$ is a {\em reproducing kernel}
of $\mathbf{H}$ if and only if:
\begin{enumerate}
\item $\forall x \in {\cal X}, \;
\kappa_x = \kappa \left (x, \cdot \right ) \in \mathbf{H}$;
\item $\forall x \in {\cal X}, \forall h \in \mathbf{H}, \;
\ps{h, \kappa_x}_\mathbf{H} = h(x)$ (reproducing property).
\end{enumerate}
A Hilbert space of functions which possesses a reproducing kernel 
is called a {\em reproducing kernel Hilbert space} (RKHS).
\end{definition}

\begin{proposition}
Let $\left ( \mathbf{H}_{\kappa}, \ps{\cdot,\cdot}_{\mathbf{H}_{\kappa}} \right)$
be a RKHS of functions on ${\cal X}$ with reproducing kernel $\kappa$.
Then, there exists a map $\Phi$ from ${\cal X}$ into a Hilbert space
$\left ( E_{\Phi \left( {\cal X} \right )}, \ps{\cdot,\cdot} \right )$ such that:
\begin{equation}
\label{eq:kernel_trick}
\forall \left ( x, x' \right ) \in {\cal X}^2, \;
\kappa \left ( x, x' \right ) = \ps{ \Phi \left ( x \right ),
\Phi \left ( x' \right )}.
\end{equation}
$\Phi$ is called a {\em feature map} and $E_{\Phi \left( {\cal X} \right )}$
a {\em feature space}.
\end{proposition}

The connection between positive semidefinite kernels and RKHS is the following.

\begin{proposition}
If $\kappa$ is a positive semidefinite kernel on ${\cal X}$, then there exists a RKHS 
$\left ( \mathbf{H}, \ps{\cdot,\cdot}_{\mathbf{H}} \right)$ of functions on $\cal X$ such
that $\kappa$ is a reproducing kernel of $\mathbf{H}$.
\end{proposition}

Let $\kappa$ be a positive semidefinite kernel on ${\cal X}$ and let
$\left ( \mathbf{H}_{\kappa}, \ps{\cdot,\cdot}_{\mathbf{H}_{\kappa}} \right)$ 
be the RKHS spanned by $\kappa$.
Let $\bar{{\cal H}} = 
\left ( \mathbf{H}_{\kappa}, \ps{\cdot,\cdot}_{\mathbf{H}_{\kappa}} \right)^Q$ and
let ${\cal H} = \left ( \left ( \mathbf{H}_{\kappa}, 
\ps{\cdot,\cdot}_{\mathbf{H}_{\kappa}} \right) + \left \{ 1 \right \} \right)^Q$.
By construction, ${\cal H}$ is the class of vector-valued functions 
$h=\left({h_{k}}\right)_{1\le k\le Q}$ on ${\cal X}$ such that
$$
h(\cdot) = \left ( \sum_{i=1}^{m_k} \beta_{ik} \kappa \left ( x_{ik}, \cdot \right ) + b_k
\right )_{1 \leq k \leq Q}
$$
where the $x_{ik}$ are elements of $\cal X$, as well as the limits of these
functions when the sets $\left \{ x_{ik}: 1 \leq i
\leq m_{k} \right \}$ become dense in ${\cal X}$
in the norm induced by the dot product (see for instance \cite{Wah99}).
Due to Equation~\ref{eq:kernel_trick}, ${\cal H}$ can be seen as a multivariate
affine model on $\Phi \left ( {\cal X} \right )$.
Functions $h$ can then be rewritten as:
$$
h(\cdot) = \left ( \ps{w_k, \cdot} + b_k \right )_{1 \leq k \leq Q}
$$
where the vectors $w_k$ are elements of
$E_{\Phi \left ( {\cal X} \right )}$. They are thus described by the pair
$\left ( \mathbf{w}, \mathbf{b} \right )$ with
$\mathbf{w} = \left ( w_k \right )_{1 \leq k \leq Q} \in
E_{\Phi \left ( {\cal X} \right )}^Q$ and
$\mathbf{b} = \left ( b_k \right )_{1 \leq k \leq Q} \in \mathbb{R}^Q$.
As a consequence, $\bar{{\cal H}}$ can be seen as a multivariate
linear model on $\Phi \left ( {\cal X} \right )$, endowed with a norm
$\|.\|_{\bar{\cal H}}$ given by:
$$
\forall \bar{h} \in \bar{{\cal H}}, \;
\left \| \bar{h} \right \|_{\bar{\cal H}} =
\sqrt{\sum_{k=1}^Q \|w_k \|^2} = \left \| \mathbf{w} \right \|,
$$
where $\| w_k \| = \sqrt{\ps{w_k, w_k}}$.
With these definitions and propositions at hand, a generic definition of the M-SVMs can be
formulated as follows.

\begin{definition}[M-SVM, Definition~42 in \cite{Gue07b}]
\label{def:M_SVM}
Let $\left ( \left ( x_i, y_i \right ) \right )_{1 \leq i \leq m} \in
\left ( {\cal X} \times \ieg 1, Q \ied \right )^m$ and $\lambda \in \mathbb{R}_+^*$.
A {\em $Q$-category M-SVM} is a large margin discriminant model
obtained by minimizing over the hyperplane $\sum_{k=1}^Q h_k = 0$ of
${\cal H}$ a penalized risk $J_{\text{M-SVM}}$ of the form:
$$
J_{\text{M-SVM}} \left ( h \right ) = \sum_{i=1}^m \ell_{\text{M-SVM}} \left ( y_i,
h \left ( x_i \right ) \right ) + 
\lambda \left \| \bar{h} \right \|_{\bar{\cal H}}^2
$$
where the data fit component
involves a loss function $\ell_{\text{M-SVM}}$ which is convex.
\end{definition}

Three main models of M-SVMs can be found in
literature. The oldest one is the model of Weston and Watkins
\cite{WesWat98}, which corresponds to the loss function $\ell_{\text{WW}}$
given by:
$$
\ell_{\text{WW}}(y,h(x))=\sum_{k\neq y} \left ( 1 - h_y(x) + h_k(x) \right )_+,
$$ 
where the {\em hinge loss} function $(\cdot)_+$ is the function $\max(0,\cdot)$.
The second one is due to Crammer and Singer \cite{CraSin01} and
corresponds to the loss function $\ell_{\text{CS}}$ given by:
$$
\ell_{\text{CS}}(y,\bar{h}(x)) = \left ( 1 - \bar{h}_y(x) + 
\max_{k\neq y}\bar{h}_k(x)\right)_+.
$$
The most recent model is the one of Lee, Lin and Wahba \cite{LeeLinWah04}
which corresponds to the loss function $\ell_{\text{LLW}}$ given by:
\begin{equation}
\label{eq:loss_LLW}
\ell_{\text{LLW}} \left ( y,h(x) \right ) = \sum _{k \neq y}
\left ( h_k(x) + \frac{1}{Q-1} \right )_{+}.
\end{equation}
Among the three models, the M-SVM of Lee, Lin and Wahba is the only one
that implements asymptotically the Bayes decision rule. It is
{\em Fisher consistent} \cite{Zha04,TewBar07}.

\subsection{The M-SVM of Lee, Lin and Wahba}

The substitution in Definition~\ref{def:M_SVM} of $\ell_{\text{M-SVM}}$
with the expression of the loss function $\ell_{\text{LLW}}$ given
by Equation~\ref{eq:loss_LLW} provides us with the expressions of
the quadratic programming (QP) problems corresponding to the training algorithms of
the hard margin and soft margin versions of the M-SVM of Lee, Lin and Wahba.

\begin{problem}[Hard margin M-SVM]
$$
\min_{\mathbf{w}, \mathbf{b}} J_{\text{HM}} \left ( \mathbf{w}, \mathbf{b} \right )
$$
$$
s.t.
\begin{cases}
\ps{w_k , \Phi (x_i)} + b_k \leq - \frac{1}{Q-1}, \;\; (1 \leq i \leq m), 
(1 \leq k \neq y_i \leq Q) \\
\sum_{k=1}^Q w_k = 0\\
\sum_{k=1}^Q b_k = 0
\\
\end{cases}
$$
where 
$$ J_{\text{HM}} \left ( \mathbf{w}, \mathbf{b} \right ) = 
\frac{1}{2} \sum_{k=1}^Q {
\| w_k \| }^2.
$$
\label{problem:primalLLW}
\end{problem}

\begin{problem}[Soft margin M-SVM]
$$
\min_{\mathbf{w}, \mathbf{b}} J_{\text{SM}}\left ( \mathbf{w}, \mathbf{b} \right )
$$
$$
s.t.
\begin{cases}
\ps{w_k , \Phi (x_i)} + b_k \leq - \frac{1}{Q-1} + \xi_{ik}, \;\; (1 \leq i \leq m), 
(1 \leq k \neq y_i \leq Q) \\
\xi_{ik} \geq 0,\;\;(1 \leq i \leq m), (1 \leq k \neq y_i \leq Q) \\
\sum_{k=1}^Q w_k = 0\\
\sum_{k=1}^Q b_k = 0
\end{cases}
$$
where $$ J_{\text{SM}} \left ( \mathbf{w}, \mathbf{b} \right ) = 
\frac{1}{2} \sum_{k=1}^Q \left \| w_k \right \|^2 + C \sum_{i=1}^m \sum_{k \neq y_i}
\xi_{ik}.
$$
\label{problem:primalSMLLW}
\end{problem}

In Problem~\ref{problem:primalSMLLW},
the $\xi_{ik}$ are {\em slack variables} introduced in order to relax
the constraints of correct classification. The coefficient $C$, which
characterizes the trade-off between prediction accuracy on the training set
and smoothness of the solution, can be expressed in terms of the regularization
coefficient $\lambda$ as follows: $ C = (2 \lambda)^{-1}$. It is called
the {\em soft margin parameter}.
Instead of directly solving Problems~\ref{problem:primalLLW}
and \ref{problem:primalSMLLW}, one usually solves their Wolfe dual \cite{Fle87}.
We now derive the dual problem of Problem~\ref{problem:primalLLW}.
Giving the details of the implementation of the Lagrangian duality
will provide us with partial results which will prove useful in the sequel.

Let $\alpha = \left ( \alpha_{ik} \right )_{1 \leq i \leq m, 1 \leq k \leq Q}
\in \mathbb{R}_+^{Qm}$ be the vector of Lagrange multipliers
associated with the constraints of good classification.
It is for convenience of notation that this vector is expressed with double subscript
and that the dummy variables $\alpha_{iy_i}$, all equal to $0$, are introduced.
Let $\delta \in E_{\Phi \left ( {\cal X} \right )}$ be the Lagrange multiplier
associated with the constraint $\sum_{k=1}^Q w_k = 0$ and $\beta \in \mathbb{R}$
the Lagrange multiplier associated with the constraint $\sum_{k=1}^Q b_k = 0$.
The Lagrangian function of Problem~\ref{problem:primalLLW} is given by:

$$
L \left ( \mathbf{w}, \mathbf{b}, \alpha, \beta, \delta \right ) = 
$$
\begin{equation}
\frac{1}{2} \sum_{k=1}^Q \| w_k \|^2 -
\ps{\delta, \sum_{k=1}^Q w_k} -
\beta \sum_{k=1}^Q b_k + \sum_{i=1}^m \sum_{k=1}^Q \alpha_{ik}
\left ( \ps{w_k , \Phi (x_i)} + b_k + \frac{1}{Q-1} \right).
\label{eq:lagrangianMC}
\end{equation}
Setting the gradient of the Lagrangian function with respect to $w_k$ equal to
the null vector
provides us with $Q$ alternative expressions for the optimal value of vector $\delta$:
\begin{equation}
\delta^* = w_k^* + \sum_{i=1}^m \alpha_{ik}^* \Phi(x_i), \;\; (1 \leq k \leq Q).
\label{eq:grad_w}
\end{equation}
Since by hypothesis, $\sum_{k=1}^Q w_k^* = 0$, summing over the index $k$
provides us with the expression of $\delta^*$ as a function of dual variables only:
\begin{equation}
\delta^* = \frac{1}{Q} \sum_{i=1}^m \sum_{k=1}^Q \alpha_{ik}^* \Phi(x_i).
\label{eq:delta_opt}
\end{equation}
By substitution into (\ref{eq:grad_w}), we get the expression of the vectors
$w_k$ at the optimum:
$$
w_k^* =\frac{1}{Q} \sum_{i=1}^m \sum_{l=1}^Q \alpha_{il}^* \Phi(x_i) -
\sum_{i=1}^m \alpha_{ik}^* \Phi(x_i), \;\; (1 \leq k \leq Q)
$$
which can also be written as
\begin{equation}
w_k^* = \sum_{i=1}^m \sum_{l=1}^Q \alpha_{il}^*
\left ( \frac{1}{Q} - \delta_{k,l} \right) \Phi(x_i), \;\; (1 \leq k \leq Q)
\label{eq:w_k}
\end{equation}
where $\delta$ is the Kronecker symbol.

Let us now set the gradient of (\ref{eq:lagrangianMC}) with respect
to $\mathbf{b}$ equal to the null vector. It comes:
$$
\beta^* = \sum_{i=1}^{m} \alpha_{ik}^*, \;\;(1 \leq k \leq Q)
\label{gradB_MC}
$$
and thus
$$
\sum_{i=1}^m \sum_{l=1}^Q \alpha_{il}^*
\left( \frac{1}{Q} - \delta_{k,l} \right) = 0, \;\; (1 \leq k \leq Q).
$$
Given the constraint $\sum_{k=1}^Q b_k=0$, this implies that:
\begin{equation}
\sum_{i=1}^m \sum_{k=1}^Q \alpha_{ik}^* b_k^* = \beta^* \sum_{k=1}^Q b_k^* = 0.
\label{eq:delta}
\end{equation}
By application of~(\ref{eq:w_k}),
$$
\sum_{k=1}^Q \left \| w_k^* \right \|^2 = 
\sum_{k=1}^Q \ps{ \sum_{i=1}^m \sum_{l=1}^Q \alpha_{il}^*
\left ( \frac{1}{Q} - \delta_{k,l} \right ) \Phi(x_i),
\sum_{j=1}^m \sum_{n=1}^Q
\alpha_{jn}^* \left ( \frac{1}{Q} - \delta_{k,n} \right ) \Phi(x_j) }
$$
$$
= \sum_{i=1}^m \sum_{j=1}^m \sum_{l=1}^Q \sum_{n=1}^Q
\alpha_{il}^* \alpha_{jn}^*  \ps{\Phi(x_i),\Phi(x_j)}
\sum_{k=1}^Q \left ( \frac{1}{Q} - \delta_{k,l} \right )
\left ( \frac{1}{Q}-\delta_{k,n} \right )
$$
\begin{equation}
\label{eq:yann}
= \sum_{i=1}^m \sum_{j=1}^m \sum_{l=1}^Q \sum_{n=1}^Q \alpha_{il}^* \alpha_{jn}^*
\left ( \delta_{l,n} - \frac{1}{Q} \right ) \kappa(x_i,x_j).
\end{equation}
Still by application of~(\ref{eq:w_k}),
$$
\sum_{i=1}^m \sum_{k=1}^Q \alpha_{ik}^* \ps{w_k^* , \Phi (x_i)} =
\sum_{i=1}^m \sum_{k=1}^Q \alpha_{ik}^*
\ps{ \sum_{j=1}^m \sum_{l=1}^Q \alpha_{jl}^*
\left ( \frac{1}{Q} - \delta_{k,l} \right) \Phi(x_j), \Phi (x_i)}
$$
\begin{equation}
\label{eq:yann2}
= \sum_{i=1}^m \sum_{j=1}^m \sum_{k=1}^Q \sum_{l=1}^Q \alpha_{ik}^* \alpha_{jl}^*
\left ( \frac{1}{Q} - \delta_{k,l} \right) \kappa(x_i,x_j).
\end{equation}
Combining (\ref{eq:yann}) and (\ref{eq:yann2}) gives:
$$
\frac{1}{2} \sum_{k=1}^Q \left \| w_k^* \right \|^2 +
\sum_{i=1}^m \sum_{k=1}^Q \alpha_{ik}^* \ps{w_k^* , \Phi (x_i)} =
- \frac{1}{2} \sum_{k=1}^Q \left \| w_k^* \right \|^2
$$
\begin{equation}
\label{eq:yann3}
= - \frac{1}{2} \sum_{i=1}^m \sum_{j=1}^m \sum_{k=1}^Q \sum_{l=1}^Q
\alpha_{ik}^* \alpha_{jl}^* \left ( \delta_{k,l} - \frac{1}{Q} \right) \kappa(x_i,x_j).
\end{equation}
In what follows, we use the notation $e_n$ to designate the vector of $\mathbb{R}^n$
such that all its components are equal to $e$. Let $H$ be the matrix
of ${\cal M}_{Qm,Qm} \left ( \mathbb{R} \right )$ of general term:
$$
h_{ik,jl} = \left ( \delta_{k,l} -\frac{1}{Q} \right ) \kappa(x_i,x_j).
$$ 
With these notations at hand,
reporting (\ref{eq:delta}) and (\ref{eq:yann3}) in (\ref{eq:lagrangianMC})
provides us with the algebraic expression of the Lagrangian function at the optimum:
$$
L \left ( \alpha^* \right ) = - \frac{1}{2} {\alpha^*}^T H \alpha^*
+\frac{1}{Q-1}1_{Qm}^T \alpha^*.
$$
This eventually provides us with the Wolfe dual formulation of
Problem~\ref{problem:primalLLW}:

\begin{problem}[Hard margin M-SVM, dual formulation]
\label{problem:dualMC}
$$
\max_{\alpha} J_{\text{LLW,d}} ( \alpha )
$$
$$
s.t. 
\begin{cases}
\alpha_{ik} \geq 0, \;\;  (1 \leq i \leq m), (1 \leq k \neq y_i \leq Q)\\
\sum_{i=1}^m \sum_{l=1}^Q \alpha_{il}
\left( \frac{1}{Q} - \delta_{k,l} \right ) = 0, \;\; (1 \leq k \leq Q)
\end{cases}
$$
where
$$
J_{\text{LLW,d}}(\alpha) = - \frac{1}{2} \alpha^T H \alpha
+ \frac{1}{Q-1} 1_{Qm}^T \alpha,
$$
with the general term of the Hessian matrix $H$ being
$$
h_{ik,jl} = \left ( \delta_{k,l} -\frac{1}{Q} \right ) \kappa(x_i,x_j).
$$
\end{problem}

Let the couple $\left ( \mathbf{w}^0, \mathbf{b}^0 \right )$
denote the optimal solution of Problem~\ref{problem:primalLLW}
and equivalently, let
$\alpha^0 = \left ( \alpha_{ik}^0 \right )_{1 \leq i \leq m, 1 \leq k \leq Q}
\in \mathbb{R}_+^{Qm}$ 
be the optimal solution of Problem~\ref{problem:dualMC}. According to
(\ref{eq:w_k}), the expression of $w_k^0$ is then:
$$
w_k^0 = \sum_{i=1}^m \sum_{l=1}^Q \alpha_{il}^0
\left ( \frac{1}{Q} - \delta_{k,l} \right ) \Phi(x_i).
$$


\subsection{Geometrical margins}
\label{sec:multi_class_margins}

From a geometrical point of view, the algorithms described
above tend to construct a set of hyperplanes
$\left \{ \left ( w_k, b_k \right ): 1 \leq k \leq Q \right\}$
that maximize globally the $C_Q^2$ {\em margins} between the differents categories.
If these margins are defined as in the bi-class case, their analytical
expression is more complex.

\begin{definition}[Geometrical margins, Definition~7 in \cite{Gue07a}]
Let us consider a $Q$-category M-SVM (a function of ${\cal H}$) classifying
the examples of its training set
$\left \{ \left ( x_i, y_i \right ): 1 \leq i \leq m \right \}$ without error.
$\gamma_{kl}$, its {\em margin between categories $k$ and $l$}, is
defined as the smallest distance of a
point either in $k$ or $l$ to the hyperplane separating those
categories. Let us denote
$$
d_{\text{M-SVM}} = \min_{1 \leq k < l \leq Q} \left \{
\min \left [ \min_{i: y_i = k}  \left (
h_k(x_i) - h_l(x_i) \right ), \min_{j: y_j = l}
\left ( h_l(x_j) - h_k(x_j) \right ) \right ] \right \}
$$
and for $1\leq k<l\leq Q$, let $d_{\text{M-SVM}, kl}$ be:
$$
d_{\text{M-SVM}, kl} = \frac{1}{d_{\text{M-SVM}}} \min \left [ \min_{i: y_i = k}
\left ( h_k(x_i) - h_l(x_i) - d_{\text{M-SVM}} \right ), \min_{j: y_j = l}
\left ( h_l(x_j) - h_k(x_j)
- d_{\text{M-SVM}} \right ) \right ].
$$
Then we have:
$$
\gamma_{kl} = d_{\text{M-SVM}} \frac{1+d_{\text{M-SVM}, kl}}{\| w_k - w_l \|}.
$$
\label{def:margins}
\end{definition}
Given the constraints of Problem~\ref{problem:primalLLW},
the expression of $d_{\text{M-SVM}}$ corresponding to the M-SVM of Lee, Lin and Wahba is:
$$
d_{\text{LLW}} = \frac{Q}{Q-1}.
$$

\begin{remark}
The values of the parameters $d_{\text{M-SVM}, kl}$ 
(or $d_{\text{LLW},kl}$ in the case of interest)
are known as soon as the pair $\left ( \mathbf{w}^0, \mathbf{b}^0 \right )$ is known.
\end{remark}

The connection between the geometrical margins and the penalizer of
$J_{\text{M-SVM}}$ is given by the following equation:
\begin{equation}
\sum_{k<l}{\| w_k - w_l \|}^2 = Q  \sum_{k=1}^Q \| w_k\|^2,
\label{sumwl}
\end{equation}
the proof of which can for instance be found in Chapter~2 of \cite{Gue07a}.
We introduce now a result needed in the proof of the master theorem of this report.
\begin{proposition}
\label{prop:primal_dual}
For the hard margin M-SVM of Lee, Lin and Wahba, we have:
$$
\frac{Q}{(Q-1)^2} \sum_{k<l}
\left ( \frac{1 + d_{\text{LLW}, kl}}{\gamma_{kl}} \right )^2 =
\sum_{k=1}^Q \| w_k^0 \|^2 =
{\alpha^0}^T H \alpha^0 =
\frac{1}{Q-1} 1_{Qm}^T \alpha^0.
$$
\end{proposition}

\begin{proof}

\begin{itemize}

\item[$\bullet$] $\frac{Q}{(Q-1)^2} \sum_{k<l}
\left ( \frac{1 + d_{\text{LLW}, kl}}{\gamma_{kl}} \right )^2 =
\sum_{k=1}^Q \| w_k^0 \|^2$\\

This equation is a direct consequence of Definition~\ref{def:margins} and
Equation~\ref{sumwl}.

\item[$\bullet$] $\sum_{k=1}^Q \| w_k^0 \|^2 = {\alpha^0}^T H \alpha^0$\\

This is a direct consequence of Equation~\ref{eq:yann3} and the definition
of matrix $H$.

\item ${\alpha^0}^T H \alpha^0 = \frac{1}{Q-1} 1_{Qm}^T \alpha^0$\\

One of the Kuhn-Tucker optimality conditions is:
$$ 
\alpha_{ik}^0 \left ( \ps{w_k^0, \Phi(x_i)} + b_k^0 + \frac{1}{Q-1}\right ) = 0, \;\;
(1 \leq i \leq m), (1 \leq k \neq y_i \leq Q),
$$
and thus:
$$
\sum_{i=1}^m \sum_{k=1}^Q \alpha_{ik}^0 \left ( \ps{w_k^0, \Phi(x_i)} + b_k^0 + 
\frac{1}{Q-1}\right ) = 0. 
$$
By application of (\ref{eq:delta}), this simplifies into
$$
\sum_{i=1}^m \sum_{k=1}^Q \alpha_{ik}^0 \ps{w_k^0, \Phi(x_i)}
+ \frac{1}{Q-1} 1_{Qm}^T \alpha^0 = 0.
$$
Since
$$
\sum_{i=1}^m \sum_{k=1}^Q \alpha_{ik}^0 \ps{w_k^0, \Phi(x_i)} =
- {\alpha^0}^T H \alpha^0 
$$
is a direct consequence of (\ref{eq:yann3}), this concludes the proof.
\end{itemize}
\end{proof}

\newpage


\section{The $\text{M-SVM}^2$}
\label{sec:SUR}

\subsection{Quadratic loss multi-class SVMs: motivation and principle}
The M-SVMs presented in Section~\ref{sec:architecture}
share a common feature with the standard pattern recognition SVM:
the contribution of the slack variables
to their objective functions is linear. Let $\xi$ be the vector of
these variables. In the cases of the M-SVMs of
Weston and Watkins and Lee, Lin and Wahba, we have
$\xi = \left ( \xi_{ik} \right )_{1 \leq i \leq m, 1 \leq k \leq Q}$
with $\left ( \xi_{iy_i} \right )_{1 \leq i \leq m} = 0_m$, and in the case
of the model of Crammer and Singer, it is simply
$\xi = \left ( \xi_i \right )_{1 \leq i \leq m}$. In both cases, the
contribution to the objective function is $C \| \xi \|_1$.

In the bi-class case, there exists a variant of the standard SVM
which is known as the {\em $2$-norm SVM} since for this machine, 
the empirical contribution
to the objective function is $C \| \xi \|_2^2$. Its main advantage,
underlined for instance in the Chapter~7 of \cite{ShaCri04}, is that its training
algorithm can be expressed, after an appropriate
change of kernel, as the training algorithm of a hard margin machine.
As a consequence, its leave-one-out error can be upper bounded thanks
to the radius-margin bound.

Unfortunately, a naive extension of the $2$-norm SVM to the multi-class case,
resulting from substituting in the objective function of either of the three M-SVMs
$\| \xi \|_1$ with $\| \xi \|_2^2$, does not preserve this property.
Section~2.4.1.4 of
\cite{Gue07a} gives detailed explanations about that point.
The strategy that we propose to exhibit interesting multi-class generalizations
of the $2$-norm SVM consists in studying the class of {\em quadratic loss M-SVMs},
i.e., the class of extensions of the M-SVMs such that the contribution of
the slack variables is a quadratic form:
$$
C \xi^T M \xi = C \sum_{i=1}^m \sum_{j=1}^m \sum_{k=1}^Q \sum_{l=1}^Q
m_{ik,jl} \xi_{ik} \xi_{jl}
$$
where $M = \left ( m_{ik,jl} \right )_{1 \leq i,j \leq m, 1 \leq k,l \leq Q}$ 
is a symmetric positive semidefinite matrix.

\subsection{The $\text{M-SVM}^2$ as a multi-class generalization of the $2$-norm SVM}

In this section, we establish that the idea introduced above provides
us with a solution to the problem of interest when the M-SVM used is the one
of Lee, Lin and Wahba and the general term of the matrix $M$ is
$m_{ik,jl} = \left ( \delta_{k,l} - \frac{1}{Q} \right ) \delta_{i,j}$.
The corresponding machine, named $\text{M-SVM}^2$, 
generalizes the $2$-norm SVM to an arbitrary
(but finite) number of categories.

\begin{problem}[$\text{M-SVM}^2$]
$$
\min_{\mathbf{w}, \mathbf{b}} J_{\text{M-SVM}^2} (\mathbf{w}, \mathbf{b})
$$
$$
s.t.
\begin{cases}
\ps{w_k, \Phi (x_i)} + b_k \leq - \frac{1}{Q-1} + \xi_{ik}, 
\;\; (1 \leq i \leq m), (1 \leq k \neq y_i \leq Q) \\
\sum_{k=1}^Q w_k = 0 \\
\sum_{k=1}^Q b_k = 0
\end{cases}
$$
where $$ J_{\text{M-SVM}^2}(\mathbf{w}, \mathbf{b}) = 
\frac{1}{2} \sum_{k=1}^Q \| w_k \|^2 +
C \sum_{i=1}^m \sum_{j=1}^m \sum_{k=1}^Q \sum_{l=1}^Q
\left ( \delta_{k,l} -\frac{1}{Q} \right ) \delta_{i,j}
\xi_{ik} \xi_{jl}.
$$
\label{problem:primalMSVM2}
\end{problem}
Note that as in the bi-class case, it is useless to introduce
nonnegativity constraints for the slack variables.
The Lagrangian function associated with Problem~\ref{problem:primalMSVM2} is thus
$$
L \left ( \mathbf{w}, \mathbf{b}, \xi, \alpha, \beta, \delta \right ) =
$$
$$
\frac{1}{2} \sum_{k=1}^Q \| w_k \|^2 +
C \xi^T M \xi -
\ps{\delta, \sum_{k=1}^Q w_k} - \beta \sum_{k=1}^Q b_k
$$
\begin{equation}
+ \sum_{i=1}^m \sum_{k=1}^Q \alpha_{ik} \left (
\ps{w_k , \Phi(x_i)} + b_k + \frac{1}{Q-1} - \xi_{ik} \right ).
\label{lagrangianMSVM2}
\end{equation}
Setting the gradient of $L$ with respect to $\xi$ equal to
the null vector gives
\begin{equation}
2C M \xi^* = \alpha^*
\label{alphaikMSVM2}
\end{equation}
which has for immediate consequence that
\begin{equation}
C {\xi^*}^T M \xi^* - {\alpha^*}^T \xi^* =
-C {\xi^*}^T M \xi^*.
\label{simplif}
\end{equation}
Using the same reasoning that we used to derive the objective
function of Problem \ref{problem:dualMC} and (\ref{simplif}),
at the optimum, (\ref{lagrangianMSVM2}) simplifies into:
\begin{equation}
L \left ( \xi^*, \alpha^* \right ) =
-\frac{1}{2} {\alpha^*}^T H {\alpha^*}
- C {\xi^*}^T M \xi^*
+ \frac{1}{Q-1} 1_{Qm}^T \alpha^*.
\label{lagrangianMSVM2-1}
\end{equation}
Besides, using (\ref{alphaikMSVM2}),
$$
\alpha_{in}^* \alpha_{ip}^* = 4C^2
\sum_{k=1}^Q \left ( \delta_{k,n}- \frac{1}{Q} \right ) \xi_{ik}^*
\sum_{l=1}^Q \left ( \delta_{l,p}- \frac{1}{Q} \right ) \xi_{il}^*
$$
and thus
$$
\alpha_{in}^* \alpha_{ip}^* = 4C^2
\sum_{k=1}^Q \sum_{l=1}^Q \left (
\delta_{k,n}\delta_{l,p} - (\delta_{k,n} + \delta_{l,p}) \frac{1}{Q} + \frac{1}{Q^2}
\right ) \xi_{ik}^* \xi_{il}^*.
$$
By a double summation over $n$ and $p$, we have:
$$
\sum_{n=1}^Q \sum_{p=1}^Q \alpha_{in}^* \alpha_{ip}^*
\left ( \delta_{n,p} - \frac{1}{Q} \right ) =
4C^2 \sum_{k=1}^Q \sum_{l=1}^Q \xi_{ik}^* \xi_{il}^*
\sum_{n=1}^Q \sum_{p=1}^Q
\left (
\delta_{k,n}\delta_{l,p} - (\delta_{k,n} + \delta_{l,p}) \frac{1}{Q} + \frac{1}{Q^2}
\right )
\left ( \delta_{n,p} - \frac{1}{Q} \right ).
$$
Since
$$
\sum_{n=1}^Q \sum_{p=1}^Q
\left (
\delta_{k,n}\delta_{l,p} - (\delta_{k,n} + \delta_{l,p}) \frac{1}{Q} + \frac{1}{Q^2}
\right )
\left ( \delta_{n,p} - \frac{1}{Q} \right ) = \delta_{k,l} - \frac{1}{Q},
$$
this simplifies into
$$
\sum_{n=1}^Q \sum_{p=1}^Q \alpha_{in}^* \alpha_{ip}^*
\left ( \delta_{n,p} - \frac{1}{Q} \right ) = 4C^2 
\sum_{k=1}^Q \sum_{l=1}^Q \left ( \delta_{k,l} - \frac{1}{Q}\right ) 
\xi_{ik}^* \xi_{il}^*.
$$
Finally, a double summation over $i$ and $j$ implies that
$$
{\alpha^*}^T M {\alpha^*} = 4C^2 {\xi^*}^T M {\xi^*}.
$$
A substitution into (\ref{lagrangianMSVM2-1}) provides us with:
$$
L \left ( \alpha^* \right ) =
-\frac{1}{2} {\alpha^*}^T \left ( H + \frac{1}{2C} M \right ) {\alpha^*}
+ \frac{1}{Q-1} 1_{Qm}^T \alpha^*.
$$
As in the case of the hard margin version of the M-SVM of Lee, Lin and Wahba,
setting the gradient of (\ref{lagrangianMSVM2}) with respect to $\mathbf{b}$
equal to the null vector gives:
$$
\sum_{i=1}^m \sum_{l=1}^Q \alpha_{il}^*
\left( \frac{1}{Q} - \delta_{k,l} \right) = 0, \;\; (1 \leq k \leq Q).
$$
Putting things together, we obtain the following expression for the dual problem
of Problem~\ref{problem:primalMSVM2}:

\begin{problem}[$\text{M-SVM}^2$, dual formulation]
\label{problem:dualM_SVM2}
$$
\max_{\alpha} J_{\text{M-SVM}^2,d} ( \alpha )
$$
$$
s.t.
\begin{cases}
\alpha_{ik} \geq 0, \;\;  (1 \leq i \leq m), (1 \leq k \neq y_i \leq Q)\\
\sum_{i=1}^m \sum_{l=1}^Q \alpha_{il}
\left( \frac{1}{Q} - \delta_{k,l} \right ) = 0, \;\; (1 \leq k \leq Q)
\end{cases}
$$
where
$$
J_{\text{M-SVM}^2,d}(\alpha) = 
- \frac{1}{2} \alpha^T \left ( H + \frac{1}{2C} M \right ) \alpha
+\frac{1}{Q-1}1_{Qm}^T \alpha.
$$
\end{problem}

Due to the definitions of the matrices $H$ and $M$,
this is precisely Problem~\ref{problem:dualMC}
with the kernel $\kappa$ replaced by a kernel $\kappa'$ such that:
$$
\kappa'(x_i, x_j) = \kappa(x_i, x_j) + \frac{1}{2C} \delta_{i,j}, \;\;
(1 \leq i,j \leq m).
$$
When $Q=2$, the M-SVM of Lee, Lin and Wahba, like the two other ones, 
is equivalent to the standard bi-class SVM (see for instance \cite{Gue07a}).
Furthermore, in that case, we get $\xi^T M \xi = \frac{1}{2} \| \xi \|_2^2$.
The $\text{M-SVM}^2$ is thus equivalent to the $2$-norm SVM.

\newpage


\section{Multi-Class Radius-Margin Bound on the Leave-One-Out Error of 
the $\text{M-SVM}^2$}
\label{sec:RM_bound}

To begin with, we must recall Vapnik's initial
bi-class theorem (see Chapter~10 of \cite{Vap98}), which is based on an intermediate
result of central importance known as the ``key lemma''.

\subsection{Bi-class radius-margin bound}

\begin{lemma}[Bi-class key lemma]
\label{lemma:keyBC}
Let us consider a hard margin bi-class SVM on a domain
${\cal X}$. Suppose that it is trained on a set
$d_m = \left\{(x_i, y_i): 1\leq i\leq m  \right\}$ of $m$
couples of ${\cal X} \times \left \{ -1, 1 \right \}$ (the points
of which it separates without error). Consider now the same
machine, trained on $d_m \setminus \left \{ (x_p, y_p) \right \}$.
If it makes an error on $(x_p, y_p)$, then the inequality
$$
\alpha_p^0 \geq \frac{1}{{\cal D}_m^2}
$$
holds, where ${\cal D}_m$ is the diameter of the smallest sphere containing
the images by the feature map of the support vectors of the initial machine.
\end{lemma}

\begin{theorem}[Bi-class radius-margin bound]
\label{theorem:ccl} 
Let $\gamma$ be the geometrical margin of
the hard margin SVM defined in Lemma~\ref{lemma:keyBC},
when trained on $d_m$. Let also
$\mathcal{L}_m$ be the number of errors resulting from applying a
leave-one-out cross-validation procedure to this machine. We have:
$$
\mathcal{L}_m \leq \frac {{\cal D}_m^2}{{\gamma}^2}.
$$
\end{theorem}

The multi-class radius-margin bound that we propose in this report
is a direct generalization of the one proposed by Vapnik.
The first step of the proof consists in establishing
a ``multi-class key lemma''. This is the subject of the following subsection.

\subsection{Multi-class key lemma}

\begin{lemma}[Multi-class key lemma]
\label{lemma:keyMC}
Let us consider a $Q$-category hard margin M-SVM of Lee, Lin and Wahba
on a domain ${\cal X}$. Let
$d_m = \left\{(x_i, y_i): 1\leq i\leq m  \right\}$
be its training set. Consider now the same machine
trained on $d_m \setminus \left \{ (x_p, y_p) \right \}$.
If it makes an error on $(x_p, y_p)$, then the inequality
$$
\max_{k \in \ieg 1, Q \ied} \alpha^0_{pk} \geq \frac{1}{Q(Q-1) {\cal D}_m^2}
$$
holds, where ${\cal D}_m$ is the diameter of the smallest sphere 
of the feature space containing the set $\left\{\Phi(x_i): 1\leq i\leq m  \right\}$.
\end{lemma}

\begin{proof}
Let $\left(\mathbf{w}^p, \mathbf{b}^p\right)$ be the couple
characterizing the optimal hyperplanes when the machine is trained on
$d_m \setminus \left \{ (x_p, y_p) \right \}$. Let
$$
\alpha^p = (\alpha_{11}^p, \ldots, \alpha_{\left (p-1 \right )
Q}^p, 0, \ldots, 0, \alpha_{\left (p+1 \right )1}^p, \ldots,
\alpha_{mQ}^p)^T
$$
be the corresponding vector of dual variables.
$\alpha^p$ belongs to $\mathbb{R}_+^{Qm}$, with 
$\left ( \alpha_{pk}^p \right )_{1 \leq k \leq Q} = 0_Q$.
This representation is used to characterize
directly the second M-SVM with respect to the first one. Indeed,
$\alpha^p$ is an optimal solution of Problem~\ref{problem:dualMC}
under the additional constraint $\left ( \alpha_{pk} \right )_{1 \leq k \leq Q} = 0_Q$.
Let us define two more vectors in $\mathbb{R}_+^{Qm}$,
$\lambda^p = (\lambda_{ik}^p)_{1 \leq i \leq m, 1 \leq k \leq Q}$ and 
$\mu^p = (\mu_{ik}^p)_{1 \leq i \leq m, 1 \leq k \leq Q}$.
$\lambda^p$ satisfies additional properties so that the vector
$\alpha^0 - \lambda^p$ is a feasible solution of
Problem~\ref{problem:dualMC} under the additional constraint that
$\left ( \alpha_{pk}^0 - \lambda_{pk}^p \right )_{1 \leq k \leq Q} = 0_Q$,
i.e., $\alpha^0 - \lambda^p$ satisfies the same
constraints as $\alpha^p$. We have
$$
\forall i \neq p, \forall k \neq y_i, \;\; \alpha_{ik}^0 -
\lambda_{ik}^p \geq 0 \Longleftrightarrow
\lambda_{ik}^p \leq \alpha_{ik}^0.
$$
We deduce from the equality constraints of Problem~\ref{problem:dualMC} that:
$$
\forall k, \;\; \sum_{i=1}^m \sum_{l=1}^Q \left (\alpha_{il}^0 -
\lambda_{il}^p \right )\left( \frac{1}{Q} - \delta_{k,l} \right ) = 0
\Longleftrightarrow
\sum_{i=1}^m \sum_{l=1}^Q \lambda_{il}^p \left( \frac{1}{Q} - \delta_{k,l} \right) = 0.
$$
To sum up, vector $\lambda^p$ satisfies the following constraints:
\begin{equation}
\begin{cases}
\forall k, \;\; \lambda_{pk}^p = \alpha_{pk}^0 \\
\forall i \neq p, \forall k, \;\; 0 \leq \lambda_{ik}^p \leq \alpha_{ik}^0 \\
\sum_{i=1}^m \sum_{l=1}^Q \lambda_{il}^p \left( \frac{1}{Q} - \delta_{k,l} \right) = 0,
\;\; (1 \leq k \leq Q)
\end{cases}.
\label{eq:cstrts3}
\end{equation}
The properties of vector $\mu^p$ are such that $\alpha^p + K_1 \mu^p$ satisfies the
constraints of the same problem, where $K_1$ is a positive scalar the value of which will
be specified in the sequel. We have thus:
$$
\forall i, \;\; \alpha_{iy_i}^p + K_1 \mu_{iy_i}^p=0
\Longleftrightarrow \mu_{iy_i}^p=0.
$$
Moreover, we have
$$
\forall i, \forall k \neq y_i, \;\; \mu_{ik}^p  \geq 0
\Longrightarrow \alpha_{ik}^p + K_1 \mu_{ik}^p \geq 0.
$$
Finally,
$$
\sum_{i=1}^m \sum_{l=1}^Q \left (\alpha_{il}^p + c \mu_{il}^p
\right ) \left( \frac{1}{Q} - \delta_{k,l} \right)  = 0
\Longleftrightarrow
\sum_{i=1}^m \sum_{l=1}^Q\mu_{il}^p \left( \frac{1}{Q} - \delta_{k,l} \right) = 0.
$$
To sum up, vector $\mu^p$ satisfies the following constraints:
\begin{equation}
\label{eq:cstrts4}
\begin{cases}
\forall i, \;\; \mu_{iy_i}^p = 0\\
\forall i, \forall k \neq y_i, \;\; \mu_{ik}^p \geq 0\\
\sum_{i=1}^m \sum_{l=1}^Q\mu_{il}^p \left( \frac{1}{Q} - \delta_{k,l} \right) = 0,
\;\; (1 \leq k \leq Q)\\
\end{cases}.
\end{equation}
In the sequel, for the sake of simplicity, we write $J$ in place of
$J_{\text{LLW,d}}$. By construction of vectors $\lambda^p$ and $\mu^p$, we have
$J(\alpha^0 - \lambda^p) \leq J(\alpha^p)$ and
$J \left ( \alpha^p + K_1 \mu^p \right ) \leq J(\alpha^0)$, and by way of
consequence,
\begin{equation}
\label{eq:lower_upper_boundMC}
J(\alpha^0) -  J(\alpha^0 - \lambda^p) \geq J(\alpha^0) -  J(\alpha^p)
\geq J \left ( \alpha^p + K_1 \mu^p \right ) - J(\alpha^p).
\end{equation}
The expression of the first term is
\begin{equation}
\label{eq:lower_upper_boundMC2}
J(\alpha^0) -  J(\alpha^0 - \lambda^p) = \frac{1}{2} {\lambda^p}^T H \lambda^p +
\left ( - H \alpha^0 + \frac{1}{Q-1} 1_{Qm} \right )^T \lambda^p.
\end{equation}
Given (\ref{eq:w_k}) and the definition of matrix $H$,
$$
\left ( - H \alpha^0 + \frac{1}{Q-1}1_{Qm} \right )^T \lambda^p =
\sum_{i=1}^m \sum_{k \neq y_i} \left ( \ps{w_k^0, \Phi (x_i)} + \frac{1}{Q-1} \right ) 
\lambda_{ik}^p
$$
\begin{equation}
\label{simplifi}
=
\sum_{i=1}^m \sum_{k \neq y_i} \left ( h_k^0 \left ( x_i \right ) + \frac{1}{Q-1} \right )
\lambda_{ik}^p - \sum_{i=1}^m \sum_{k \neq y_i} b_k^0 \lambda_{ik}^p.
\end{equation}
Due to the constraints of correct classification and the nonnegativity of the components
of vector $\lambda^p$, the first double sum
of the right-hand side of (\ref{simplifi}) is nonpositive.
Furthermore, making use of the equality constraints of (\ref{eq:cstrts3})
and $\sum_{k = 1}^Q b_k^0=0$ gives:
$$
\sum_{i=1}^m \sum_{k=1}^Q b_k^0 \lambda_{ik}^p =
\sum_{k=1}^Q b_k^0 \sum_{i=1}^m \lambda_{ik}^p =
\left ( \sum_{k=1}^Q b_k^0 \right) \left(\sum_{i=1}^m \sum_{l = 1}^Q \frac{1}{Q}
\lambda_{il}^p \right) = 0.
$$
Thus,
$$
\left ( - H \alpha^0 + \frac{1}{Q-1} 1_{Qm} \right )^T \lambda^p \leq 0.
$$
A substitution into (\ref{eq:lower_upper_boundMC2}) provides us with
the following upper bound on $J(\alpha^0) - J(\alpha^0 - \lambda^p)$:
$$
J(\alpha^0) - J(\alpha^0 - \lambda^p) \leq \frac{1}{2} {\lambda^p}^T H \lambda^p,
$$
and equivalently, by definition of $H$,
\begin{equation}
\label{eq:left_hand_sideMC}
J(\alpha^0) - J(\alpha^0 - \lambda^p) \leq
\frac{1}{2} \sum_{k=1}^Q \left \|
\sum_{i=1}^m \sum_{l=1}^Q \lambda_{il}^p \left( \frac{1}{Q} - \delta_{k,l} \right)
\Phi(x_i) \right \|^2.
\end{equation}
We now turn to the right-hand side of
(\ref{eq:lower_upper_boundMC}). The line of reasoning already used
for the left-hand side gives:
$$
J \left ( \alpha^p + K_1 \mu^p \right ) - J(\alpha^p) =
$$
\begin{equation}
\label{bleu}
K_1 \left ( - H \alpha^p + \frac{1}{Q-1}1_{Qm} \right )^T \mu^p -
\frac{K_1^2}{2} \sum_{k=1}^Q \left \| \sum_{i=1}^m \sum_{l=1}^Q
\mu_{il}^p \left( \frac{1}{Q} - \delta_{k,l} \right)  \Phi(x_i) \right \|^2
\end{equation}
with
\begin{equation}
\label{eq:part01}
\left ( - H \alpha^p + \frac{1}{Q-1}1_{Qm} \right )^T \mu^p =
\sum_{i=1}^m \sum_{k \neq y_i} \left ( h_k^p \left ( x_i \right ) + \frac{1}{Q-1} \right )
\mu_{ik}^p.
\end{equation}
By hypothesis, the M-SVM trained on $d_m \setminus \left \{ (x_p, y_p) \right \}$
does not classify $x_p$ correctly. This means that there exists 
$n \in \ieg 1, Q \ied \setminus \left \{ y_p \right \}$ such that
$h_n^p \left ( x_p \right ) \geq 0$. Let ${\cal I}$ be a mapping from
$\ieg 1, Q \ied \setminus \left \{ n \right \}$ to
$\ieg 1, m \ied \setminus \left \{ p \right \}$ such that 
$$
\forall k \in \ieg 1, Q \ied \setminus \left \{ n \right \}, \;\;
\alpha_{{\cal I}(k)n}^p > 0. 
$$
We know that such a mapping exists, otherwise,
given the equality constraints of Problem~\ref{problem:dualMC},
vector $\alpha^p$ would be equal to the null vector.
For $K_2 \in \mathbb{R}_+^*$,
let $\mu^p$ be the vector of $\mathbb{R}^{Qm}$ that only differs from the null
vector in the following way:
$$
\begin{cases}
\mu_{pn}^p = K_2 \\
\forall k \in \ieg 1, Q \ied \setminus \left \{ n \right \}, \;\; 
\mu_{{\cal I}(k)k}^p = K_2
\end{cases}.
$$
Obviously, this solution is feasible (satisfies the constraints~\ref{eq:cstrts4}). 
Indeed, $\frac{1}{Q} \sum_{i=1}^m \sum_{k=1}^Q \mu_{ik}^p = K_2$ and
$\sum_{i=1}^m \mu_{ik}^p = K_2$, $\left ( 1 \leq k \leq Q \right )$.
With this definition of
vector $\mu^p$, the right-hand side of (\ref{eq:part01}) simplifies into:
$$
K_2 \left ( h_n^p \left ( x_p \right ) + 
\sum_{k \neq n} h_k^p \left ( x_{{\cal I}(k)} \right ) + \frac{Q}{Q-1} \right ).
$$
Vector $\mu^p$ has been specified so as to make it possible to exhibit a nontrivial
lower bound on this last expression. By definition of $n$, 
$h_n^p \left ( x_p \right ) \geq 0$. Furthermore, the Kuhn-Tucker optimality conditions:
$$
\alpha_{ik}^p \left ( \ps{w_k^p , \Phi (x_i)} + b_k^p + \frac{1}{Q-1} \right ) = 0,
\;\; (1 \leq i \neq p \leq m), (1 \leq k \neq y_i \leq Q)
$$
imply that $\left ( h_k^p \left ( x_{{\cal I}(k)} \right ) 
\right )_{1 \leq k \neq n \leq Q} = - \frac{1}{Q-1} 1_{Q-1}$.
As a consequence, a lower bound on the right-hand side of (\ref{eq:part01})
is provided by:
$$
\sum_{i=1}^m \sum_{k \neq y_i} \left ( h_k^p \left ( x_i \right ) + \frac{1}{Q-1} \right )
\mu_{ik}^p \geq
\frac{K_2}{Q-1}.
$$
It springs from this bound and (\ref{bleu}) that
\begin{equation}
\label{eq:right_hand_sideMC}
J \left ( \alpha^p + K_1 \mu^p \right ) - J(\alpha^p) \geq \frac{K_1 K_2}{Q-1} -
\frac{K_1^2}{2} \sum_{k=1}^Q \left \| \sum_{i=1}^m \sum_{l=1}^Q
\mu_{il}^p \left( \frac{1}{Q} - \delta_{k,l} \right)  \Phi(x_i) \right \|^2.
\end{equation}

Combining (\ref{eq:lower_upper_boundMC}),
(\ref{eq:left_hand_sideMC}) and (\ref{eq:right_hand_sideMC})
finally gives:
$$
\frac{1}{2} \sum_{k=1}^Q \left \|
\sum_{i=1}^m \sum_{l=1}^Q \lambda_{il}^p \left( \frac{1}{Q} - \delta_{k,l} \right)
\Phi(x_i) \right \|^2 \geq
$$
\begin{equation}
\label{eq:gap_cMC}
\frac{K_1 K_2}{Q-1} - \frac{K_1^2}{2} \sum_{k=1}^Q \left \| \sum_{i=1}^m \sum_{l=1}^Q
\mu_{il}^p \left( \frac{1}{Q} - \delta_{k,l} \right)  \Phi(x_i) \right \|^2.
\end{equation}
Let $\nu^p = (\nu_{ik}^p)_{1 \leq i \leq m, 1 \leq k \leq Q}$ be the vector of
$\mathbb{R}_+^{Qm}$ such that $\mu^p = K_2 \nu^p$.
The value of the scalar $K_3 = K_1 K_2$ maximizing the right-hand side of
(\ref{eq:gap_cMC}) is:
$$
K_3^* = \frac{ \frac{1}{Q-1}}{
\sum_{k=1}^Q \left \| \sum_{i=1}^m \sum_{l=1}^Q
\nu_{il}^p \left( \frac{1}{Q} - \delta_{k,l} \right)  \Phi(x_i) \right \|^2
}.
$$
By substitution in (\ref{eq:gap_cMC}), this means that:
$$
\label{eq:alpha_lower_boundMC}
(Q-1)^2 \sum_{k=1}^Q \left \|
\sum_{i=1}^m \sum_{l=1}^Q \lambda_{il}^p \left( \frac{1}{Q} - \delta_{k,l} \right)
\Phi(x_i) \right \|^2
\sum_{k=1}^Q \left \| \sum_{i=1}^m \sum_{l=1}^Q
\nu_{il}^p \left( \frac{1}{Q} - \delta_{k,l} \right)  \Phi(x_i) \right \|^2 \geq 1.
$$
For $\eta$ in $\mathbb{R}^{Qm}$,
let $K(\eta) = \frac{1}{Q} \sum_{i=1}^m \sum_{k=1}^Q \eta_{ik}^p$.
We have:
$$
\left \| \frac{1}{Q} \sum_{i=1}^m \sum_{l=1}^Q \lambda_{il}^p \Phi(x_i) -
\sum_{i=1}^m \lambda_{ik}^p \Phi(x_i)\right \| ^2 = 
K \left ( \lambda^p \right )^2
\left \| \text{conv}_1 (\Phi(x_i)) - \text{conv}_2 (\Phi(x_i)) \right \| ^2
$$
where $\text{conv}_1 (\Phi(x_i))$ and $\text{conv}_2 (\Phi(x_i))$ are two convex
combinations of the $\Phi(x_i)$. As a consequence,
$\left \| \text{conv}_1 (\Phi(x_i)) - \text{conv}_2 (\Phi(x_i)) \right \| ^2$
can be bounded from above by ${\cal D}_m^2$. Since the same reasoning applies to $\nu^p$,
we get:
\begin{equation}
\label{eq:with_K}
(Q-1)^2 Q^2 K \left ( \lambda^p \right )^2 K \left ( \nu^p \right )^2 {\cal D}_m^4
\geq 1.
\end{equation}
By construction, $K \left ( \nu^p \right ) = 1$. We now construct a vector
$\lambda^p$ minimizing the objective function $K$.
First, note that due to the equality constraints satisfied by this vector,
$$
\forall k \in \ieg 1, Q \ied, \;\; \sum_{i=1}^m \lambda_{ik}^p =
\frac{1}{Q} \sum_{i=1}^m \sum_{l=1}^Q \lambda_{il}^p.
$$
As a consequence,
$$
\forall (k,l) \in \ieg 1, Q \ied^2, \;\; \sum_{i=1}^m \lambda_{ik}^p =
\sum_{i=1}^m \lambda_{il}^p.
$$
This implies that:
$$
\forall k \in \ieg 1, Q \ied, \;\; \sum_{i=1}^m \lambda_{ik}^p \geq
\max_{l \in \ieg 1, Q \ied} \alpha_{pl}^0.
$$
Obviously, both the box constraints in (\ref{eq:cstrts3}) and the nature of $K$
call for the choice of small values for the components $\lambda_{ik}^p$.
Thus, there is a feasible solution ${\lambda^p}^*$ such that:
$$
\forall k \in \ieg 1, Q \ied, \;\; \sum_{i=1}^m {\lambda_{ik}^p}^* = 
\max_{l \in \ieg 1, Q \ied} \alpha_{pl}^0.
$$
This solution is such that $K \left ( {\lambda^p}^* \right ) = 
\max_{k \in \ieg 1, Q \ied} \alpha_{pk}^0$.
The substitution of the values of $K \left ( \nu^p \right )$ and 
$K \left ( {\lambda^p}^* \right )$ in (\ref{eq:with_K}) provides us with:
$$
\left ( \max_{k \in \ieg 1, Q \ied} \alpha_{pk}^0 \right )^2 \geq
\frac{1}{(Q-1)^2Q^2 {\cal D}_m^4}.
$$
Taking the square root of both sides concludes the proof of the lemma.
\end{proof}

\subsection{Multi-class radius-margin bound}

\begin{theorem}[Multi-class radius-margin bound]
\label{theorem:MCccl}
Let us consider a $Q$-category hard margin M-SVM of Lee, Lin and Wahba
on a domain ${\cal X}$. Let
$d_m = \left\{(x_i, y_i): 1\leq i\leq m  \right\}$
be its training set,
$\mathcal{L}_m$ the number of errors resulting from
applying a leave-one-out cross-validation
procedure to this machine, and ${\cal D}_m$
the diameter of the smallest sphere
of the feature space containing the set $\left\{\Phi(x_i): 1\leq i\leq m  \right\}$.
Then the following upper bound holds true:
$$
\mathcal{L}_m \leq Q^2 {\cal D}_m^2 \sum_{k<l}
\left(\frac{1+ d_{\textrm{LLW},kl}}{\gamma_{kl}}\right)^2.
$$
\end{theorem}

\begin{proof}
Lemma~\ref{lemma:keyMC} exhibits a non trivial lower bound on
$\max_{k \in \ieg 1, Q \ied} \alpha_{pk}^0$ 
when the machine trained on the set $d_m \setminus \left \{ (x_p, y_p) \right \}$
makes an error on $(x_p, y_p)$, i.e., when $(x_p, y_p)$ contributes to
$\mathcal{L}_m$. As a consequence,
\begin{equation}
\label{eq:last_link}
1_{Qm}^T \alpha^0 \geq
\sum_{i=1}^m \max_{k \in \ieg 1, Q \ied} \alpha_{ik}^0 \geq
\frac{\mathcal{L}_m}{Q(Q-1) {\cal D}_m^2}.
\end{equation}
According to Proposition~\ref{prop:primal_dual},
$1_{Qm}^T \alpha^0 = \frac{Q}{Q-1} \sum_{k<l}
\left ( \frac{1 + d_{\text{LLW}, kl}}{\gamma_{kl}} \right )^2$.
A substitution in (\ref{eq:last_link}) thus provides us with the result announced.
\end{proof}

\newpage

\section{Conclusions and Future Work}
\label{sec:conclusion}
In this report, we have introduced a variant of the M-SVM
of Lee, Lin and Wahba that strictly generalizes to the multi-class
case the $2$-norm SVM. 
For this quadratic loss M-SVM, named $\text{M-SVM}^2$,
we have then established 
a generalization of Vapnik's radius-margin bound.
We conjecture that this bound could be improved by a $Q^2$ factor.
As it is, it can already be compared with those
proposed in \cite{WanXueCha05} for model selection.
This, with a general study of the quadratic loss M-SVMs,
is the subject of an ongoing research.

\subsection*{Acknowledgements}
The work of E.~Monfrini was supported by the Decrypthon program
of the ``Association Française contre les Myopathies'' (AFM),
the CNRS and IBM.

\newpage

\tableofcontents
\newpage
\bibliography{App}

\begin{thebibliography}{10}

\bibitem{BerTho04}
A.~Berlinet and C.~Thomas-Agnan.
\newblock {\em Reproducing Kernel Hilbert Spaces in Probability and
  Statistics}.
\newblock Kluwer Academic Publishers, Boston, 2004.

\bibitem{BosGuyVap92}
B.E. Boser, I.M. Guyon, and V.N. Vapnik.
\newblock A training algorithm for optimal margin classifiers.
\newblock In {\em COLT'92}, pages 144--152, 1992.

\bibitem{ChaVapBouMuk02}
O.~Chapelle, V.N. Vapnik, O.~Bousquet, and S.~Mukherjee.
\newblock Choosing multiple parameters for support vector machines.
\newblock {\em Machine Learning}, 46(1):131--159, 2002.

\bibitem{CorVap95}
C.~Cortes and V.N. Vapnik.
\newblock Support-vector networks.
\newblock {\em Machine Learning}, 20(3):273--297, 1995.

\bibitem{CraSin01}
K.~Crammer and Y.~Singer.
\newblock On the algorithmic implementation of multiclass kernel-based vector
  machines.
\newblock {\em Journal of Machine Learning Research}, 2:265--292, 2001.

\bibitem{Fle87}
R.~Fletcher.
\newblock {\em Practical Methods of Optimization}.
\newblock John Wiley \& Sons, Chichester, second edition, 1987.

\bibitem{Gue07a}
Y.~Guermeur.
\newblock {\em SVM multiclasses, théorie et applications}.
\newblock Habilitation à diriger des recherches, UHP, 2007.
\newblock (in French).

\bibitem{Gue07b}
Y.~Guermeur.
\newblock {VC} theory of large margin multi-category classifiers.
\newblock {\em Journal of Machine Learning Research}, 8:2551--2594, 2007.

\bibitem{HasTibFri01}
T.~Hastie, R.~Tibshirani, and J.~Friedman.
\newblock {\em The Elements of Statistical Learning - Data Mining, Inference,
  and Prediction}.
\newblock Springer, New York, 2001.

\bibitem{LeeLinWah04}
Y.~Lee, Y.~Lin, and G.~Wahba.
\newblock Multicategory support vector machines: Theory and application to the
  classification of microarray data and satellite radiance data.
\newblock {\em Journal of the American Statistical Association},
  99(465):67--81, 2004.

\bibitem{LunBra69}
A.~Luntz and V.~Brailovsky.
\newblock On estimation of characters obtained in statistical procedure of
  recognition.
\newblock {\em Technicheskaya Kibernetica}, 3, 1969.
\newblock (in Russian).

\bibitem{Mas03}
P.~Massart.
\newblock Concentrations inequalities and model selection.
\newblock In {\em Ecole d'Eté de Probabilités de Saint-Flour~XXXIII}, LNM.
  Springer-Verlag, 2003.

\bibitem{ShaCri04}
J.~Shawe-Taylor and N.~Cristianini.
\newblock {\em Kernel Methods for Pattern Analysis}.
\newblock Cambridge University Press, Cambridge, 2004.

\bibitem{TewBar07}
A.~Tewari and P.L. Bartlett.
\newblock On the consistency of multiclass classification methods.
\newblock {\em Journal of Machine Learning Research}, 8:1007--1025, 2007.

\bibitem{Vap98}
V.N. Vapnik.
\newblock {\em Statistical Learning Theory}.
\newblock John Wiley \& Sons, Inc., New York, 1998.

\bibitem{VapCha00}
V.N. Vapnik and O.~Chapelle.
\newblock Bounds on error expectation for support vector machines.
\newblock {\em Neural Computation}, 12(9):2013--2036, 2000.

\bibitem{Wah99}
G.~Wahba.
\newblock Support vector machines, reproducing kernel {Hilbert} spaces, and
  randomized {GACV}.
\newblock In B.~Sch\"olkopf, C.J.C. Burges, and A.J. Smola, editors, {\em
  Advances in Kernel Methods, Support Vector Learning}, chapter~6, pages
  69--88. The MIT Press, Cambridge, MA, 1999.

\bibitem{WanXueCha05}
L.~Wang, P.~Xue, and K.L. Chan.
\newblock Generalized radius-margin bounds for model selection in multi-class
  {SVMs}.
\newblock Technical report, School of Electrical and Electronic Engineering,
  Nanyang Technological University, Singapore, 639798, 2005.

\bibitem{WesWat98}
J.~Weston and C.~Watkins.
\newblock Multi-class support vector machines.
\newblock Technical Report {CSD-TR-98-04}, {Royal Holloway, University of
  London, Department of Computer Science}, 1998.

\bibitem{Zha04}
T.~Zhang.
\newblock Statistical analysis of some multi-category large margin
  classification methods.
\newblock {\em Journal of Machine Learning Research}, 5:1225--1251, 2004.

\end{thebibliography}
\bibliographystyle{plain}

\end{document}